\documentclass[letterpaper, 10 pt, conference]{ieeeconf}  

\IEEEoverridecommandlockouts            
\let\labelindent\relax

\usepackage{enumitem}
\usepackage{color}
\usepackage{tikz,pgfplots}
\usetikzlibrary{arrows.meta}
\usepackage{amsmath,graphicx}
\usepackage{amsthm,amssymb}
\usepackage{algorithm}
\usepackage{algpseudocode}
\usepackage{bbold}
\usepackage[utf8]{inputenc}
\newtheorem{theorem}{Theorem}
\newtheorem{assumption}{Assumption}
\newtheorem{corollary}{Corollary}
\newtheorem{remark}{Remark}

\newenvironment{proof-sketch}{\noindent{\bf Sketch of Proof}\hspace*{1em}}{\qed\bigskip}
\newcommand{\asto}{\overset{\rm a.s.}{\longrightarrow}}
\usepackage{color}
\usepackage{hyperref}
\usepackage{float}
\usepackage{appendix}
\usepackage{comment}
\theoremstyle{definition}

\newcommand{\trans}{{\sf T}}

\newcommand{\tr}{{\rm tr}}
\newcommand{\RR}{{\mathbb{R}}}
\newcommand{\CC}{{\mathbb{C}}}
\usepackage[colorinlistoftodos]{todonotes}
\DeclareMathOperator{\argmin}{argmin}

\title{\LARGE \bf
Random Matrix-Improved Estimation of the Wasserstein Distance between two Centered Gaussian Distributions
}

\author{Malik Tiomoko$^{1}$, Romain Couillet$^{1,2,*}$
\\ $^{1}$CentraleSup\'elec, Universit\'e ParisSaclay,\quad $^{2}$GIPSA-lab, Universit\'e Grenoble-Alpes
\thanks{*Couillet's work is supported by the ANR Project RMT4GRAPH (ANR-14-CE28-0006) and the IDEX GSTATS Chair at University Grenoble Alpes.}
}

\begin{document}

\maketitle
\thispagestyle{empty}
\pagestyle{empty}

\begin{abstract}

This article proposes a method to consistently estimate functionals $\frac1p\sum_{i=1}^pf(\lambda_i(C_1C_2))$ of the eigenvalues of the product of two covariance matrices $C_1,C_2\in\mathbb{R}^{p\times p}$ based on the empirical estimates $\lambda_i(\hat C_1\hat C_2)$ ($\hat C_a=\frac1{n_a}\sum_{i=1}^{n_a} x_i^{(a)}x_i^{(a)\trans}$), when the size $p$ and number $n_a$ of the (zero mean) samples $x_i^{(a)}$ are similar. As a corollary, a consistent estimate of the Wasserstein distance (related to the case $f(t)=\sqrt{t}$) between centered Gaussian distributions is derived.

The new estimate is shown to largely outperform the classical sample covariance-based ``plug-in'' estimator. Based on this finding, a practical application to covariance estimation is then devised which demonstrates potentially significant performance gains with respect to state-of-the-art alternatives.

\end{abstract}

\section{Introduction}
Many machine learning and signal processing applications require an adequate framework to compare statistical objects, starting with probability distributions. The Wasserstein distance, initially inspired by Monge \cite{monge1781memoire} and later by Kantorovich \cite{kantorovich1942translocation} in a transport theory analogy, provides a natural notion of dissimilarity for probability measures and finds a wide spectrum of applications in
image analysis \cite{rubner2000earth}, shape matching \cite{su2015optimal}, computer vision \cite{ni2009local}, etc. 

However, computing the Wasserstein distance is expensive as it requires to minimize a cost function taking the form of an integral over the space of probability measures. Despite recent advances \cite{cuturi2013sinkhorn}, where regularized approximations that reduce this numerical cost are proposed, the latter is still involved in general. Special cases exist for which the Wasserstein distance assumes a closed form, particularly when the underlying distributions are zero-mean Gaussian with covariance matrices $C_1$ and $C_2$. The closed-form formula however involves the eigenvalues of $C_1C_2$ and thus depends on the unknown population covariance matrices $C_1$ and $C_2$. Assuming the observation of $n_1,n_2\gg p$ samples with covariances $C_1,C_2$, respectively, $C_1C_2$ is conventionally approximated by its empirical version $\hat{C}_1\hat{C}_2$. As we will show, this induces a dramatic estimation bias in practical applications where $p$ is rather large or, equivalently, $n_1,n_2$ rather small, a standard assumption in big data applications.

Based on recent advances in random matrix theory, this article proposes a new consistent estimate for the Wasserstein distance between two centered Gaussian distributions when the dimension $p$ of the samples is of the same order of magnitude as their numbers $n_1,n_2$. This work enters the scope of Mestre's seminal ideas \cite{MES08} on the estimation of functionals $\frac1p\sum_{i=1}^pf(\lambda_i(C))$ of the eigenvalue distribution of population covariance matrices $C$, which can be related to the (limiting) eigenvalue distribution of the sample estimates $\hat C$ via a complex integration trick. We recently extended this work to the estimation of functionals of the eigenvalue distribution of F-matrices in \cite{couillet2018random}, i.e., matrices of the form $C_1^{-1}C_2$, and applied to the estimation of the natural geodesic  Fisher distance, Battacharrya distance, and R\'enyi/Kullbach-Leibler divergences between Gaussian distributions.

Our main contribution is the extension of \cite{MES08,couillet2018random} to functionals $f$ of the eigenvalues of \emph{products} $C_1C_2$ of population covariance matrices. The Wasserstein distance falls within this scope for \mbox{$f(t)=\sqrt{t}$}. Unlike \cite{couillet2018random}, where the functionals of interest ($f(t)=t,\log(t),\log^2(t)$) are amenable to explicit evaluations of the complex integrals, the present $f(t)=\sqrt{t}$ scenario is more technically involved and gives rise to real non-explicit, yet numerically computable, integrals.

In the remainder of the article, Section~\ref{sec:model} introduces the main model and assumptions, Section~\ref{sec:main} provides our key technical result and its corollary to the Wasserstein distance estimation, and a practical application to covariance matrix estimation is finally proposed in Section~\ref{sec:applications}.

\medskip

\noindent{\bf Reproducibility.} Matlab codes for the various estimators introduced and studied in this article are available at \href{link}{https://github.com/maliktiomoko/RMTWasserstein}
\section{Model and Main Objective}
\label{sec:model}
For $a \in \{1,2\}$, let $X_a=[x_1^{(a)},\ldots,x_{n_a}^{(a)}]$ be $n_a$ independent and identically distributed random vectors with $x_i^{(a)}=C_a^{\frac12}\Tilde{x}_i^{(a)}$, where $\Tilde{x}_i^{(a)}\in\RR^p$ has zero mean, unit variance and finite fourth order moment entries. This holds in particular for $x_i^{(a)}\sim\mathcal N(0,C_a)$. In order to control the growth rates of $n_1,n_2,p$, we make the following assumption:
\begin{assumption}[Growth Rates]
\label{ass:growth rates}
As $n_a \to\infty$, $p/n_a\to c_a\in (0,1)$ and $\limsup_p\max\{\|C_a^{-1}\|,\|C_a\|\} < \infty$ for $\|\cdot\|$ the operator norm.
\end{assumption}
We define the sample covariance estimate $\hat C_a$ of $C_a$ as
\begin{align*}
    \hat C_a\equiv \frac1{n_a}X_aX_a^\trans = \frac1{n_a}\sum_{i=1}^{n_a} x_i^{(a)}x_i^{(a)\trans}.
\end{align*}

\medskip

The Wasserstein distance $D_W(C_1,C_2)$ between two zero-mean Gaussian distributions with covariances $C_1$ and $C_2$, respectively, assumes the form \cite[Remark~2.31]{PEYRE2019}:
\begin{equation}
\label{eq:Wasserstein}
D_W(C_1,C_2)=\tr(C_1)+\tr(C_2)-2\tr\left[(C_1^{\frac12}C_2C_1^{\frac12})^\frac12\right].
\end{equation}
It is easily shown that, under Assumption~\ref{ass:growth rates},
\begin{align*}
    \frac1p \tr \hat{C}_a - \frac1p\tr C_a \to 0
\end{align*}
almost surely. But estimating $\tr(C_1^{\frac12}C_2C_1^{\frac12})^\frac12$ is more involved: this is the focus of the article. Up to a normalization by $p$, this term can be written under the functional form:
\begin{equation}
\label{eq:D}
\frac1p\tr(C_1^{\frac12}C_2C_1^{\frac12})^\frac12=\frac1p \sum_{i=1}^n\!\sqrt{\lambda_i(C_1C_2)} \equiv D(C_1,C_2;\sqrt{\cdot})
\end{equation}
with $\lambda_i(X)$ the $i$-th smallest eigenvalue of $X$. 

\medskip

Our objective is to estimate the more generic form
\begin{equation}
\label{eq:Df}
    D(C_1,C_2;f) \equiv \frac1p \sum_{i=1}^n f(\lambda_i(C_1C_2))
\end{equation}
for $f:\RR\to \RR$ a real function admitting a complex-analytic extension. To this end, we shall relate the eigenvalues $\lambda_i(C_1C_2)$ to $\lambda_i(\hat C_1\hat C_2)$ through the {\it Stieltjes transform} ($m_{\theta}(z)\equiv \int \frac{d\theta(\lambda)}{\lambda-z}$ for measure $\theta$ and $z\in\CC$) of their associated normalized counting measures
\begin{align*}
\mu_p=\frac1p\sum_{i=1}^p \delta_{\lambda_i(\hat{C}_1\hat{C}_2)},\quad \nu_p=\frac1p\sum_{i=1}^p \delta_{\lambda_i(C_1C_2)}.
\end{align*}
 In particular, $m_{\mu_p}(z)=\frac1p\sum_{i=1}^p\frac1{\lambda_i-z}$ for $\lambda_i=\lambda_i(\hat C_1\hat C_2)$.

With these notations, we are in position to introduce our main results.

\section{Main results}
\label{sec:main}
The following theorem provides a consistent estimate for the metric $D(C_1,C_2;f)$ defined in \eqref{eq:Df}.
\begin{theorem}
\label{th:main}
 Let  $\Gamma \subset \{z\in\mathbb{C},{\rm real}[z]>0\}$ be a contour surrounding $\cup_{p=1}^{\infty} {\rm supp}(\mu_p)$. Then, under Assumption~\ref{ass:growth rates},
 \begin{align*}
    &D(C_1,C_2;f) - \hat D(X_1,X_2;f) \asto 0
    \end{align*}
    where
    \begin{align*}
    \hspace*{-.5cm}    \hat D(X_1,X_2;f) = \frac{n_2}{2\pi ip}\oint_{\Gamma}\!f \left(\frac{\varphi_p(z)}{\psi_p(z)}\right)\left[\frac{\varphi_{p}'(z)}{\varphi_p(z)}-\frac{\psi_p'(z)}{\psi_p(z)}\right]\psi_p(z)dz
 \end{align*}
 and, recalling $m_{\mu_p}(z)=\frac1p\sum_{i=1}^p\frac1{\lambda_i-z}$ for $\lambda_i=\lambda_i(\hat C_1\hat C_2)$, $\varphi_p(z)=\frac{z}{1-\frac p{n_1}-\frac p{n_1}zm_{\mu_p}(z)}$, $\psi_p(z)=1-\frac p{n_2}-\frac p{n_2}zm_{\mu_p}(z)$.
\end{theorem}

The result of Theorem~\ref{th:main} is very similar to \cite[Theorem~1]{couillet2018random} established for functionals of the eigenvalues of $C_1^{-1}C_2$. The main difference lies in the expression of the function $\varphi_p(z)$.

\begin{proof}
The proof of Theorem~\ref{th:main} is based on the same approach as for \cite[Theorem~1]{COU18}. One first creates a link between the Stieltjes transform $m_{\nu_p}$ and $D(C_1,C_2;f)$ using Cauchy's integral formula:
\begin{align}
\label{eq:Cauchy}
    \frac1p\sum_{i=1}^pf(\lambda_i(C_1C_2))&=\int f(t)d\nu_p(t) \nonumber \\
    &=\frac{1}{2\pi i}\int \left[\oint_{\Gamma_\nu} \frac{f(z)}{z-t} dz\right] d\nu_p(t) \nonumber \\
    &= \frac{-1}{2\pi i}\oint_{\Gamma_\nu} f(z)m_{\nu_p}(z)dz
\end{align}
with $\Gamma_\nu$ a contour surrounding the support ${\rm supp}(\nu_p)$ of $\nu_p$. To relate the unknown $m_{\nu_p}$ to the observable $m_{\mu_p}$, we proceed as follows. By first conditioning on $\hat C_1$, $\hat C_1^{\frac12}\hat C_2\hat C_1^{\frac12}$ is seen as a sample covariance matrix for the samples $\hat C_1^{\frac12}C_2^{\frac12}\tilde x_i^{(2)}$, for which \cite{SIL95} allows one to relate $m_{\mu_p}$ to the Stieltjes transform of the eigenvalue distribution $\zeta_p$ of $C_2^{\frac12}\hat C_1C_2^{\frac12}$. The latter is yet another sample covariance matrix for the samples $C_2^{\frac12}C_1^{\frac12}\tilde x_i^{(1)}$; exploiting \cite{SIL95} again creates the connection from $m_{\zeta_p}$ to $m_{\nu_p}$. This entails the two equations:
\begin{align}
    \label{eq:bai1}
    m_{\mu_p}(z) &= \varphi_p(z) m_{\zeta_p}\left( \varphi_p(z)\right) + o_p(1) \\
    \label{eq:bai2}
    m_{\nu_p}\left( \frac{z}{\Psi_p(z)} \right) &= m_{\zeta_p}(z) \Psi_p(z)+o_p(1).
\end{align}
where $\Psi_p(z)\equiv 1-\frac p{n_2}-\frac p{n_2}zm_{\zeta_p}(z)$.
Successively plugging \eqref{eq:bai1}--\eqref{eq:bai2} into \eqref{eq:Cauchy} by means of two successive appropriate changes of variables, we obtain Theorem~\ref{th:main}.
\end{proof}

\bigskip

Theorem~\ref{th:main} takes the form of a complex integral which, for generic choices of $f$, needs be numerically evaluated. In the specific case of present interest where $f(z)=\sqrt{z}$, this complex integral can be evaluated as follows.

\begin{theorem}
\label{the:estimation}
    Let $\lambda_1\leq \ldots \leq \lambda_p$, with $\lambda_i\equiv \lambda_i(\hat{C}_1\hat{C}_2)$, and define $\{\xi_i\}_{i=1}^p$ and $\{\eta_i\}_{i=1}^p$ the (increasing) eigenvalues of $\Lambda-\frac{1}{n_1}\sqrt{\lambda}\sqrt{\lambda}^\trans$ and $\Lambda-\frac{1}{n_2}\sqrt{\lambda}\sqrt{\lambda}^\trans$, respectively, where $\lambda=\left(\lambda_1,\ldots,\lambda_p\right)^\trans$, $\Lambda={\rm diag}(\lambda)$ and $\sqrt{.}$ is understood entry wise. Then, under Assumption~\ref{ass:growth rates},
    \begin{align*}
        D(C_1,C_2;\sqrt{\cdot}) - \hat D(X_1,X_2;\sqrt{\cdot}) \asto 0
    \end{align*}
    where, if $n_1\neq n_2$, 
    \begin{align*}
    \hat D(X_1,X_2;\sqrt{\cdot}) &= { 2\sqrt{n_1n_2}} \frac1p\sum_{j=1}^{p}\sqrt{\lambda_j} \\
    &+\frac{2n_2}{\pi p}\!\sum_{j=1}^p \int_{\xi_j}^{\eta_j}\!\sqrt{-\frac{\varphi_p(x)}{\psi_p(x)}}\psi_p'(x) dx
    \end{align*}
    with $\varphi_p,\psi_p$ defined in Theorem~\ref{th:main} and, if $n_1=n_2$,
    \begin{equation*}
        \hat D(X_1,X_2;\sqrt{\cdot})=\frac{2n_1}{p}\sum_{j=1}^p\left(\sqrt{\lambda_j}-\sqrt{\xi_j}\right).
    \end{equation*}
\end{theorem}
While still assuming an integral form (when $n_1\neq n_2$), this formulation no longer requires the arbitrary choice of a contour $\Gamma$ and significantly reduces the computational time to estimate $D(C_1,C_2,\sqrt{\cdot})$. For $n_1=n_2$, a case of utmost practical interest, the expression is completely explicit and computationally only requires to evaluate the eigenvalues $\xi_j$ of $\Lambda-\frac{1}{n_1}\sqrt{\lambda}\sqrt{\lambda}^\trans$. The latter being a (negative definite) rank-$1$ perturbation of $\Lambda$, by Weyl's interlacing lemma \cite{franklin2012matrix}, the $\xi_j$'s are interlaced with the $\lambda_j$'s as
\begin{align*}
    \xi_1\leq \lambda_1 \leq \xi_2 \leq \ldots \leq \xi_p \leq \lambda_p.
\end{align*}
As the $\lambda_j$'s are of order $O(1)$ with respect to $p$, $|\lambda_j-\xi_j|\leq |\lambda_j-\lambda_{j-1}|=O(p^{-1})$, therefore explaining why the expression of $\hat D(X_1,X_2;\sqrt{\cdot})$ is of order $O(1)$.



\begin{proof}
The $\xi_i$ and $\eta_i$, as defined in the theorem statement, are the respective zeros of the rational functions $1-\frac p{n_1}-\frac p{n_1}zm_{\tilde{\mu}_p}(z)$ and $1-\frac p{n_2}-\frac p{n_2}zm_{\tilde{\mu}_p}(z)$ (see \cite[Appendix~B]{COU18}). Thus, $\varphi_p$ and $\psi_p$ can be expressed under the rational form:
\begin{equation*}
    \varphi_p(z)=z\frac{\prod_{i=1}^pz-\lambda_i}{\prod_{i=1}^pz-\eta_i}, \quad \psi_p(z)=\frac{\prod_{i=1}^pz-\xi_i}{\prod_{i=1}^pz-\lambda_i}.
\end{equation*}
\begin{figure}
    \centering
    \includegraphics[scale=0.3]{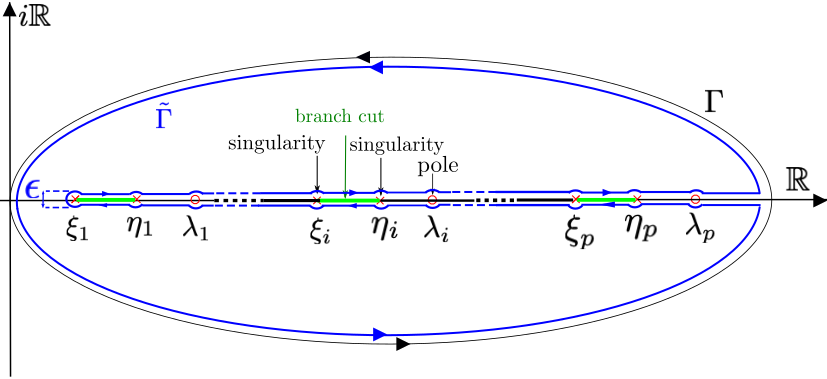}
    \caption{Deformation of the initial contour $\Gamma$ (in black) into the new contour $\tilde{\Gamma}$ (in blue). The branch cuts are represented in green (i.e., real $z$'s for which the argument of  $\varphi(z)\psi(z)$ is negative). 
    }
    \vspace{-0.5cm}
    \label{fig:new_contour}
\end{figure}
Evaluating the estimate from Theorem~\ref{th:main} for $f(z)=\sqrt{z}$ then requires to evaluate a complex integral involving rational functions and square roots of rational functions. Since the complex square root is multivalued, a careful control of ``branch-cuts'' is required. To perform this calculus, we deform the integration contour $\Gamma$ of Theorem~\ref{th:main} into $\tilde\Gamma$ as per Figure~\ref{fig:new_contour}. In the case $n_1\neq n_2$, the closed null-integral contour $\tilde\Gamma$ (blue in Figure~\ref{fig:new_contour}) is the sum of the sought-for integral over $\Gamma$ and of four extra components:
\begin{enumerate}[leftmargin=*]
    \item Integrals over $\epsilon$-radius circles around $\xi_i$: those are null in the limit $\epsilon\to 0$, as confirmed by a change of variable $z=\xi_i+\epsilon e^{\imath\theta}$ which allows one to bound the integrand;
    \item Integrals over the real axis (in the $\epsilon\to 0$ limit):
    \begin{align*}
    &A_2=\frac{n_2}{\pi p}\sum_{j=1}^p \int_{\xi_j+\epsilon}^{\eta_j-\epsilon}\sqrt{-(\varphi_p\psi_p)(x)}\left[2\frac{\psi_p^{'}(z)}{\psi_p(z)}\right.\\
    &-\left.\left(\frac{\varphi_p^{'}(z)}{\varphi_p(z)}+\frac{\psi_p^{'}(z)}{\psi_p(z)}\right)\right]dx\\
        &=\frac{2n_2}{\pi p}\sum_{j=1}^p \int_{\xi_j}^{\eta_j}\sqrt{-\varphi_p\psi_p(x)}\left[\frac{\psi_p^{'}(z)}{\psi_p(z)}\right]dx\\
        &-\frac{n_2}{\pi p}\sum_{j=1}^p \int_{\xi_j+\epsilon}^{\eta_j-\epsilon}\frac{\sqrt{-\varphi_p\psi_p(x)}}{\varphi_p\psi_p(x)}\left[\frac{d}{dx}\left(\varphi_p(x)\psi_p(x)\right)\right]dx\\
        &=\frac{2n_2}{\pi p}\!\sum_{j=1}^p \int_{\xi_j}^{\eta_j}\!\sqrt{-\frac{\varphi_p(x)}{\psi_p(x)}}\psi_p'(x) dx\\
        &-2\frac{n_2}{\pi p}\sum_{j=1}^p \frac{1}{\sqrt{\epsilon\frac{d}{dx}\left(\frac{1}{(\varphi_p\psi_p(x))}\right)(\eta_j)}}+o(\epsilon)
    \end{align*}
    
    \item Integrals over the $\epsilon$-radius circles around $\eta_j$, with $\epsilon\to 0$
    \begin{equation*}
        A_3=2\frac{n_2}{\pi p}\sum_{j=1}^p \frac{1}{\sqrt{\epsilon\frac{d}{dx}\left(\frac{1}{(\varphi_p\psi_p(x))}\right)(\eta_j)}}+o(\epsilon)
    \end{equation*}
    which thus compensates the last ($\epsilon$-diverging) term in $A_2$.
    \item Residues in the $\lambda_j$ poles
    \begin{align*}
        A_4&=2\frac{n_2}{p}\lim_{z\to\lambda_j}\sum_{j=1}^{p}\sqrt{(\varphi_p\psi_p)(z)}=2\frac{n_2}{p}\sqrt{\frac{n_1}{n_2}}\sum_{j=1}^p\sqrt{\lambda_j}.
    \end{align*}
\end{enumerate}
Putting these terms together entails the result of the theorem for the case where $n_1\neq n_2$.
For $n_1=n_2$, it suffices to take the limit of the expression as $\xi_j\to \eta_j$. This yields:
\begin{align*}
    \hat D(X_1,X_2;\sqrt{\cdot}) &= \frac{2n_1}p\sum_{j=1}^{p}\sqrt{\lambda_j} \\
    &+\frac{2n_1}{p}\!\sum_{j=1}^p \frac{1}{\pi}\lim_{t\rightarrow \xi_j}\int_{\xi_j}^{t}\!\sqrt{-\frac{\varphi_p(x)}{\psi_p(x)}}\psi_p'(x) dx\\
    &= \frac{2n_1}p\sum_{j=1}^{p}\sqrt{\lambda_j} \\
    &-\frac{2n_1}{p}\!\sum_{j=1}^p \frac{1}{2\pi\imath}\lim_{\epsilon\rightarrow 0}\oint_{\Gamma_{\xi_j}^{\epsilon}}\!\sqrt{-\varphi_p\psi_p(x)}\frac{\psi_p'(x)}{\psi_p(x)} dx
\end{align*}
where $\Gamma_{\xi_j}^{\epsilon}$ is an $\epsilon$-radius circular contour around $\xi_j$.
The second equality is obtained by deforming the real integral in the complex plane (see \cite{saff2003fundamentals} for complex analysis details). 
The result unfolds by letting ${z=\xi_i+\epsilon e^{\imath\theta}}$.


\end{proof}

\bigskip

Consequently, we obtain the following $n,p$-consistent estimate for the Wasserstein distance $D_W(C_1,C_2)$ of \eqref{eq:Wasserstein}.
\begin{corollary}[Consistent Estimate of $D_W(C_1,C_2)$]
\label{cor:estimate}
    Under Assumption~\ref{ass:growth rates},
    \begin{align}
        \frac1pD_W(C_1,C_2) - \left[\frac1p\tr (\hat C_1+\hat C_2) - 2\hat D(X_1,X_2;\sqrt{\cdot}) \right] \asto 0
    \end{align}
    for $\hat D(X_1,X_2;\sqrt{\cdot})$ given by Theorem~\ref{the:estimation}.
\end{corollary}
\begin{remark}[Estimation of $\|C_1-C_2\|_{F}^{2}$]
The Frobenius distance between two covariance matrices also falls under the scope of the present article for the function $f(z)=z$. Indeed,
\begin{align*}
    D_F(C_1,C_2)&=\|C_1-C_2\|_{F}^{2}=\tr\left(C_1^2+C_2^2\right)-2\tr\left(C_1C_2\right).
\end{align*}
Then under Assumption~\ref{ass:growth rates} and along with the fact that $\frac1p\tr C_1^2$ can be estimated consistently from $\frac1p\tr\hat C_1^2-\frac1{n_1p}(\tr \hat C_1)^2$,
\begin{align*}
        \frac1pD_F(C_1,C_2) - &\left[\frac1{p}\tr (\hat C_1^2+\hat C_2^2)-\frac{p}{n_1}\left(\frac1p\tr\hat{C}_1 \right)^2\right.\\
        &\left.-\frac{p}{n_2} \left(\frac1p\tr\hat{C}_2 \right)^2- 2\hat D(X_1,X_2;\cdot) \right] \asto 0.
    \end{align*}
In this case, $\hat D(X_1,X_2;\cdot)$ assumes the simple expression
\begin{align*}
    \hat D(X_1,X_2;\cdot)&=\frac 1p \sum_{j=1}^p \lambda_j = \frac1p\tr \hat C_1\hat C_2
\end{align*}
 which follows from $\frac1p\tr\hat C_1\hat C_2-\frac1p\tr C_1C_2\asto 0$ (by elementary probability arguments) or equivalently from a residue calculus based on Theorem~\ref{th:main} for $f(z)=z$.
\end{remark}

\section{Simulations and Applications}
\label{sec:applications}
In this section, we first corroborate our theoretical findings by comparing the classical plug-in estimator to our proposed estimator on synthetic Gaussian data. We then provide an application of our results to improved covariance matrix estimation based on few samples.

\subsection{Confirmation of our results on synthetic data}
We here compare the classical plug-in estimate of the Wasserstein distance (that is \eqref{eq:Wasserstein} with $C_a$ replaced by $\hat C_a$, $a=1,2$) with our proposed estimate in Corollary~\ref{cor:estimate}. Table~\ref{tab:1} lists the results obtained for Toeplitz matrices $C_1,C_2$ estimated based on various values of $p,n_1,n_2$. While our proposed estimator is designed under a large $p,n_1,n_2$ assumption (as per Assumption~\ref{ass:growth rates}), it achieves competitive performances even for small values of $p$, corroborating here our findings in \cite{couillet2018random} for other classes of covariance matrix distances.

\begin{table}[h!]
	\centering
	
	\begin{tabular}{r|rrr}
		$p$ & $D_{\rm W}(C_1,C_2)$ & Classical  & Proposed  \\
		\hline
		2 & 0.0110  &0.0127   &0.0120   \\
		4 & 0.0175 &0.0198   &0.0183   \\
		8 &0.0208  &0.0232  &0.0206   \\
		16 &0.0225  &0.0280  &0.0227   \\
		32 &0.0233  &0.0339  &0.0234  \\
		64 &0.0237  &{\bf 0.0451}  &0.0240  \\
		128& 0.0239  &{\bf \underline{0.0667}}  &0.0244  \\
		256&0.0240  &\bf \underline{\underline{0.1092}}  &0.0244  \\
		512& 0.0241  &\bf \underline{\underline{0.1953}}  & 0.0245
	\end{tabular}
	
	\medskip 
	{ (error $<$ 5\%)} 
	{\bf (error $>$ 50\%)}	
	{\bf \underline{(error $>$ 100\%)}}
	{\bf \underline{\underline{(error $>$ 300\%)}}}


	\caption{Estimators of the Wasserstein distance between $C_1$ and $C_2$ with $[C_1]_{ij}=.2^{|i-j|}$, $[C_2]_{ij}=.4^{|i-j|}$, $x_i^{(a)}\sim \mathcal N(0,C_a)$; $n_1=1024$ and $n_2=2048$ for different $p$. Averaged over $100$ trials.}	
	
\label{tab:1}
\label{tab:comparison}
\end{table}
\vspace{-.7cm}

\subsection{Application to covariance matrix estimation}

As a concrete application, Theorem~\ref{th:main} may be used to improve the actual estimation of covariance matrices under a small number $n\sim p$ of sample data, as similarly performed in \cite{tiomoko2019random} for other covariance matrix distances. 

The idea is as follows: we first particularize Theorem~\ref{th:main} and Theorem~\ref{the:estimation} to the case where one of the covariance matrices, say $C_1$, is known by taking $c_1=0$ (i.e., $n_1\to\infty$ for all fixed $p$). This gives access to estimates for $D_W(M,C_2;\sqrt{\cdot})$ for all deterministic positive definite matrix $M$. We then minimize this estimated distance over $M$ in order to estimate $C_2$ by means of a gradient descent approach.

For $C_1$ known, we redefine ${\mu_p=\frac{1}{p}\sum_{i=1}^p\delta_{\lambda_i(C_1\hat{C}_2)}}$ and obtain, as a corollary of Theorem~\ref{th:main}:
\begin{theorem}
\label{the:known_matrix}
Let  $\Gamma \subset \{z \in \mathbb{C},{\rm real}[z]>0\}$ a contour surrounding $\cup_{p=1}^{\infty} {\rm supp}(\mu_p)$. Then,
 \begin{align*}
    D(C_1,C_2;f)-\frac{1}{2\pi ic_2}\oint_{\Gamma}F\left(-m_{\tilde{\mu}_p}(z)\right)dz \asto 0
 \end{align*}
 with $m_{\tilde{\mu}_p}(z)=\frac{p}{n_2}m_{{\mu}_p}(z)+\frac{p-n_2}{n_2z}$ and $F'(z)=f(\frac 1z)$.
\end{theorem}
\begin{proof}
For $C_1$ known ($c_1 \to 0$), $\varphi_p(z)=z$, and the estimator of Theorem~\ref{th:main} yields:
\begin{align*}
    \hat D(X_1,X_2;f) = \frac{1}{2\pi i}\oint_{\Gamma}\!f \left(\frac{z}{\psi_p(z)}\right)\left[\frac{\psi_p'(z)}{\psi_p(z)}-\frac{1}{z}\right]\frac{\psi_p(z)dz}{c_2}.
\end{align*}
Using the relation $m_{\tilde{\mu}_p}(z)=-\frac{\psi_p(z)}{z}$, we then get
\begin{align*}
    \hat D(X_1,X_2;f) = -\frac{1}{2\pi ic_2}\oint_{\Gamma}\!f \left(-\frac{1}{m_{\tilde{\mu}_p}(z)}\right)m_{\tilde{\mu}_p}'(z)zdz
\end{align*}
and the result is immediate after an integration by parts.
\end{proof}

\bigskip

For $f(z)=\sqrt{z}$, one has $F(z)=2\sqrt{z}$ and we obtain, with a similar proof as for Theorem~\ref{the:estimation},
\begin{align*}
    &D(C_1,C_2;\sqrt{\cdot})- \hat D(C_1,X_2;\sqrt{\cdot}) \asto 0, \\
    &\hat D(C_1,X_2;\sqrt{\cdot})=\frac{2}{\pi c_2}\sum_{j=1}^p \int_{\xi_j}^{\lambda_j}\sqrt{m_{\tilde{\mu}_p(x)}}dx.
\end{align*}

Our objective is now to exploit the fact that
\begin{equation}
\label{optim_initial}
    C_2=\argmin_{M\succ 0} D_W(M,C_2)
\end{equation}
where the minimization is over the open cone of positive definite matrices. Using the approximation $D(M,C_2;\sqrt{\cdot})\simeq \hat D(M,X_2;\sqrt{\cdot})$, we are then tempted to minimize $\frac1p\tr (M+\hat C_2)-2\hat D(M,X_2;\sqrt{\cdot})$ in place of $D_W(M,C_2)$. The former quantity however has a non zero probability to be negative, and we thus instead propose to estimate $C_2$ as:
\begin{align*}
    \check C_2&=\argmin_{M} h(M) \\
    \quad h(M)&=\left[\frac1p\tr (M+\hat C_2)-2\hat D(M,X_2;\sqrt{\cdot})\right]^2.
\end{align*}

To compute the gradient $\nabla h(M)$ of $h$ at position $M$, one needs to evaluate the differential ${\rm D}h(M)[\xi]$, at $M$ and in the direction $\xi$, in the Riemmanian manifold of $p\times p$ symmetric positive definite matrices (see  \cite{absil2009optimization,tiomoko2019random} to further technical details). We then use the relation ${\rm D}h(M)[\xi]=\langle \nabla h(M),\xi\rangle_{M}$ where $\langle\cdot,\cdot\rangle_{.}$ is the Riemmanian metric defined as $\langle\eta,\xi\rangle_{M}=\tr\left(M^{-1}\eta M^{-1} \xi \right)$. We obtain the relation
\begin{align*}
    &\pi\imath p\frac{\nabla h(M)}{2\sqrt{h(M)}}=\frac1pM^2\\
    &+\sum_{j=1}^p\int_{\xi_j}^{\lambda_j}\sqrt{\frac{1}{m_{\tilde{\mu}_p}(x)}}{\rm sym}\left(M\hat{C}_2(M\hat{C}_2-xI_p)^{-2}M\right) dx
\end{align*}
where ${\rm sym}(A)=\frac12 (A+A^\trans)$ is the symmetric part of $A \in \mathbb{R}^{p\times p}$.
We can write the latter as:
\begin{equation*}
    \nabla h(M)=2\sqrt{h(M)}\left[{\rm sym}\left(V\Lambda_{\nabla}V^{-1}\right)+\frac 1p M^2\right]
\end{equation*}
where $V$ is the orthogonal matrix of the eigenvectors of $M\hat{C}_2$ and $\Lambda_{\nabla}$ is the diagonal matrix with
\begin{align*}
    \left[\Lambda_{\nabla}\right]_{kk}&=\frac{1}{\pi p}\sum_{j\neq k}\int_{\xi_j}^{\lambda_j}\sqrt{\frac{1}{m_{\tilde{\mu}_p}(x)}}\frac{1}{(\lambda_k-x)^2}dx\\
    &+\frac{1}{\pi p}\sum_{j\neq k}\int_{\xi_k}^{\lambda_k}\sqrt{\frac{1}{m_{\tilde{\mu}_p}(x)}}\frac{1}{(\lambda_j-x)^2}dx.
\end{align*}
This finally entails the gradient descent Algorithm~\ref{alg:estimC}.


\begin{algorithm}
{\bf Require} Positive definite initialization $M=M_0$.

\smallskip

{\bf Repeat} {$M \gets M^{\frac12} \exp\left(-t M^{-\frac12} \nabla h(M) M^{-\frac12} \right) M^{\frac12}$ with $t$ either fixed or optimized by backtracking line search.}

{\bf Until} {Convergence.}

\smallskip

{{\bf Return} $M$.}
\caption{Proposed estimation algorithm.}
\label{alg:estimC}
\end{algorithm}
Figure~\ref{fig:covariance} depicts the results of the algorithm. There is displayed the Wasserstein distance $D_W(C,\cdot)$ between a matrix $C$ having four distinct eigenvalues of equal multiplicity (precisely, $\nu_p=\frac14(\delta_{.1}+\delta_3+\delta_4+\delta_5)$) and various estimators of $C$: the sample covariance matrix (SCM), the state-of-the-art ``non-linear shrinkage'' estimators QuEST1 \cite{LW15} (based on a Frobenius distance minimization) and QuEST2 \cite{LW18} (based on a Stein loss minimization), and the result of the gradient descent approach proposed in this section. For fair comparison, the iterative QuEST1, QuEST2 and our proposed method are all initialized at $M_0$ the linear shrinkage estimator from \cite{ledoit2004well}. Note that our proposed choice of $C$ is particularly suited to mimick an ``optimal transport'' problem of displacing the eigenvalues of $M_0$ to the discrete four positions of the eigenvalues of $C$.

In addition to the computational simplicity of our gradient-descent approach with respect to the QuEST estimators (see the numerical method details in \cite{ledoit2017numerical}), the figure demonstrates significant gains brought by our proposed approach for large values of $p/n$, where the SCM particularly fails.

\section{Concluding Remarks}
Interestingly, while the Fisher distance or Kullbach-Liebler divergence, which depend on logarithms of \emph{inverse} of covariance matrices, are understandably difficult to estimate in the $n_1,n_2<p$ regime (see \cite{COU18} for advanced discussions on this matter), the Wasserstein distance should not be confronted with this limitation. Yet, the invertibility of $C_1,C_2$ and the request for $c_1,c_2\in(0,1)$ (i.e., $p<n_1,n_2$) from Assumption~\ref{ass:growth rates} are fundamental to our proofs. Precisely, the variable changes exploited in the proof of Theorem~\ref{th:main} to reach a contour $\Gamma_\nu$ correctly surrounding ${\rm supp}(\nu_p)$ from a contour $\Gamma$ surrounding ${\rm supp}(\mu_p)$ are not satisfying if $c_1>1$ or $c_2>1$. These surprising difficulties need clarification.
 
Another point of interest lies in the comparative advantage of exploiting a particular covariance matrix distance in specific scenarios. For instance, it may seem that ill-conditioned matrices should be more tolerated by Wasserstein distance estimators than by Fisher distance estimators. Yet, this aspect is not obvious in our proofs and also deserves more insights.
 
\section*{Acknowledgement}
We thank Pedro Rodrigues for helpful discussions and references on optimal transport.

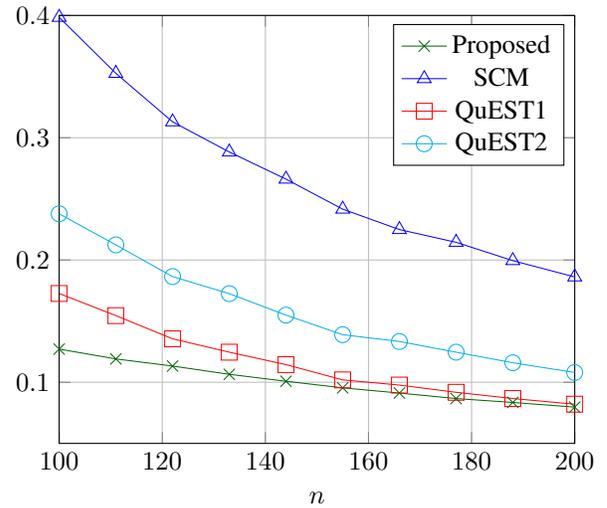
\begin{figure}[t]
    \centering
    \vspace{-0.2cm}
    \begin{tikzpicture}
        \begin{axis}[xlabel=$n$,grid=major,xmin=100,xmax=200,ymin=.05,ymax=.4]

        \addplot[ultra thin,mark size=3pt,mark=x,green!40!black]coordinates{
 (100,1.271253e-01)(111,1.190826e-01)(122,1.132466e-01)(133,1.064960e-01)(144,1.008248e-01)(155,9.542498e-02)(166,9.117903e-02)(177,8.666517e-02)(188,8.349900e-02)(200,7.968029e-02)
};
\addplot[ultra thin,mark size=3pt,mark=triangle,blue]coordinates{
    (100,3.984692e-01)(111,3.526451e-01)(122,3.130032e-01)(133,2.883282e-01)(144,2.661358e-01)(155,2.415792e-01)(166,2.248902e-01)(177,2.144673e-01)(188,1.994804e-01)(200,1.861496e-01)
};        

\addplot[ultra thin,mark size=3pt,mark=square,red]coordinates{
                    (100,1.725999e-01)(111,1.544890e-01)(122,1.355330e-01)(133,1.245952e-01)(144,1.144260e-01)(155,1.019346e-01)(166,9.769291e-02)(177,9.178613e-02)(188,8.670012e-02)(200,8.206614e-02)
};
\addplot[ultra thin,mark size=3pt,mark=o,cyan]coordinates{
                    (100,2.377855e-01)(111,2.123607e-01)(122,1.864354e-01)(133,1.724844e-01)(144,1.549034e-01)(155,1.389033e-01)(166,1.333437e-01)(177,1.245966e-01)(188,1.159382e-01)(200,1.079776e-01)
};
\legend{Proposed,SCM,QuEST1,QuEST2}
        \end{axis}
    \end{tikzpicture}
    \vspace{-0.25cm}
    \caption{Wasserstein distance $D_W(C,\cdot)$ between $C$ with $\nu_p=\frac14(\delta_{.1}+\delta_3+\delta_4+\delta_5)$ and (green) our proposed estimator, (blue) the sample covariance matrix, (red) and (light blue) the QuEST estimators proposed in \cite{LW18,LW15}; for $p=100$ and varying number of samples $n$ averaged over 10 realizations. 
    }
    \label{fig:covariance}
    \vspace{-0.5cm}
\end{figure}

\bibliographystyle{IEEEbib.bst}
\bibliography{reference.bib}
\end{document}